\documentclass[twoside]{article}

\usepackage{amssymb}
\usepackage{amsmath}
\usepackage{dsfont}
\usepackage{enumitem}
\usepackage{amsmath}
\usepackage{amsfonts}
\usepackage{mathtools}
\usepackage{algorithm,algorithmic}
\usepackage{bm}
\usepackage{multirow}

\usepackage{subcaption}

\newcommand{\R}[0]{\mathbb{R}}
\newcommand{\bx}[0]{\bm{x}}
\newcommand{\pvert}[0]{\;\vert\;}
\DeclareMathOperator{\argmin}{arg\,min}
\DeclareMathOperator{\argmax}{arg\,max}

\usepackage{algorithm,algorithmic}
\usepackage{amsthm}
\newtheorem{proposition}{Proposition}[section]
\newtheorem{theorem}{Theorem}[section]

\newtheorem{lemma}[theorem]{Lemma}

\usepackage[accepted]{aistats2022}
\usepackage[round]{natbib}

\begin{document}
\twocolumn[

\aistatstitle{Scalable Bayesian Transformed Gaussian Processes}

\aistatsauthor{ Xinran Zhu \And Leo Huang \And  Cameron Ibrahim \And Eric Hans Lee \And David Bindel}

\aistatsaddress{ Cornell University \And Cornell University  \And University of Delaware \And SigOpt \And Cornell University} 
]

\begin{abstract}
The Bayesian transformed Gaussian process (BTG) model, proposed by Kedem and Oliviera, is a fully Bayesian counterpart to the warped Gaussian process (WGP) and marginalizes out a joint prior over input warping and kernel hyperparameters. 
This fully Bayesian treatment of hyperparameters often provides more accurate regression estimates and superior uncertainty propagation, but is prohibitively expensive. 
The BTG posterior predictive distribution, itself estimated through high-dimensional integration, must be inverted in order to perform model prediction. 
To make the Bayesian approach practical and comparable in speed to maximum-likelihood estimation (MLE), we propose principled and fast techniques for computing with BTG. Our framework uses doubly sparse quadrature rules, tight quantile bounds, and rank-one matrix algebra to enable both fast model prediction and model selection. 
These scalable methods allow us to regress over higher-dimensional datasets and apply BTG with layered transformations that greatly improve its expressibility. We demonstrate that BTG achieves superior empirical performance over MLE-based models. 
\end{abstract}

\section{Introduction \label{sec:intro}}
Gaussian processes (GPs) provide a powerful probabilistic learning framework, including a marginal
likelihood which represents the probability of data given only GP hyperparameters. The marginal
likelihood automatically balances model fit and complexity terms to favor the simplest models that
explain the data. 

A GP assumes normally distributed observations. In practice, however, this condition is not always adequately met. 
The classic approach to moderate departures from normality is trans-Gaussian kriging, which applies a normalizing nonlinear transformation to the data~\citep{Cressie93}. This idea was reprised and expanded upon in the machine learning literature. One instance is the warped GP (WGP), which maps the observation space to a latent space in which the data is well-modeled by a GP and which learns GP hyperparameters through maximum likelihood estimation \citep{WGP}. The WGP paper employs a class of parametrized, hyperbolic tangent transformations. Later, \citet{CWGP} introduced compositionally warped GPs (CWGP), which chain together a sequence of parametric transformations with closed form inverses. Bayesian warped GPs further generalize WGPs by modelling the transformation as a GP \citep{BWGP}. These are in turn generalized to Deep GPs by \cite{DGP}, which stack GPs in the layers of a neural network.

Throughout this line of work, the GP transformation and kernel hyperparameters are typically learned through joint maximum likelihood estimation (MLE). A known drawback of MLE is overconfidence in the data-sparse or low-data regime, which may be exacerbated by warping \citep{chai2019improving}. Bayesian approaches, on the other hand, offer a way to account for uncertainty in values of model parameters. 

Bayesian trans-kriging \citep{spock2009bayesian} treats both transformation and kernel parameters in a Bayesian fashion. A prototypical Bayesian trans-kriging model is the BTG model developed by \citet{BTG}. The model places an uninformative prior on the precision hyperparameter and analytically marginalizes it out to obtain a posterior distribution that is a mixture of Student's t-distributions. Then, it uses a numerical integration scheme to marginalize out transformation and remaining kernel parameters. In this latter regard, BTG is consistent with other Bayesian methods in the literature, including those of \cite{gibbs1998bayesian, adams2009tractable, lalchand2020approximate}. 
While BTG shows improved prediction accuracy and better uncertainty propagation, it comes with several computational challenges, which hinder its scalability and limit its competitiveness with the MLE approach.

First, the cost of numerical integration in BTG scales with the dimension of hyperparameter space, which can be large when transforms and noise model parameters are incorporated. Traditional methods such as Monte Carlo (MC) suffer from slow convergence. As such, we leverage sparse grid quadrature and quasi Monte Carlo (QMC), which have a higher degree of precision but require a sufficiently smooth integrand. Second, the posterior mean of BTG is not guaranteed to exist, hence the need to use the posterior \textit{median} predictor. The posterior median and credible intervals do not generally have closed forms, so one must resort to expensive numerical root-finding to compute them. Finally, while fast cross-validation schemes are known for vanilla GP models, leave-one-out-cross-validation (LOOCV) on BTG, which incurs quartic cost naively, is less straightforward to perform because of an embedded generalized least squares problem. 

In this paper, we reduce the overall computational cost of end-to-end BTG inference, including model prediction and selection. Our main contributions follow.
\begin{itemize}
    \item We propose efficient and scalable methods for computing BTG predictive medians and quantiles through a combination of doubly sparse quadrature and quantile bounds. We also propose fast LOOCV using rank-one matrix algebra.
    \item We develop a framework to control the tradeoff between speed and accuracy for BTG and analyze the error in sparsifying QMC and sparse grid quadrature rules.
    \item We empirically compare the Bayesian and MLE approaches and provide experimental results for BTG and WGP coupled with 1-layer and 2-layer transformations. We find evidence that BTG is well-suited for low-data regimes, where hyperparameters are under-specified by the data.
    \item We develop a modular Julia package for computing with transformed GPs (e.g., BTG and WGP) which exploits vectorized linear algebra operations and  supports MLE and Bayesian inference.
\end{itemize}

\section{Background \label{sec:background}}
\subsection{Gaussian Process Regression} 
A GP $f \sim \mathcal{GP}(\mu,\tau^{-1} k)$ is a distribution over functions in $\mathbb{R}^d$, where $\mu( x)$ is the expected value of $f( x)$ and $\tau^{-1} k( x,  x')$ is the positive (semi)-definite covariance between $f( x)$ and $f( x')$. 
For later clarity, we separate the precision hyperparameter $\tau$ from lengthscales and other kernel hyperparameters (typically denoted by $\theta$).

Unless otherwise specified, we assume a linear mean field and the squared exponential kernel: 
\begin{align*}
    \mu_\beta(\bx) &= {\mathbf \beta}^T m(\bx), \qquad m\colon \R^d \to \R^p, \\
    k_\theta(\bx, \mathbf x') &= \exp\Big(-\frac{1}{2}\lVert \bx - \mathbf x' \rVert^2_{\mathbf{D}^{-2}_\theta}\Big).
\end{align*}
Here $m$ is a known function mapping a location to a vector of covariates, $\beta$ consists of coefficients in the linear combination, and ${D}^{2}_\theta$ is a diagonal matrix of length scales determined by the parameter(s) $\theta$.

For any finite set of input locations, let:
\begin{align*}
    X &= \lbrack \bx_1,\ldots,\bx_n\rbrack^T & X &\in \R^{n\times d},\\
    {M}_X &= \lbrack m(\bx_1), \ldots, m(\bx_n) \rbrack^T & {M}_X &\in \R^{n\times p},\\
    {f}_X &= \lbrack f(\bx_1),\ldots, f(\bx_n)\rbrack^T & {f}_X &\in \R^{n},
\end{align*}
where $X$ is the matrix of observations locations, $M_X$ is the matrix of covariates at $X$, and ${f}_X$ is the vector of observations. A GP has the property that any finite number of evaluations of $f$ will have a joint Gaussian distribution:
${f}_X\pvert { \beta}, \tau, \theta \sim \mathcal{N}({M}_X{ \beta}, \tau^{-1}{K}_{X})$, where $(\tau^{-1} {K}_{X})_{ij} = \tau^{-1} k_\theta(\bx_i, \bx_j)$ is the covariance matrix of ${f}_X$.  We assume ${M}_X$ to be full rank.

The posterior predictive density of a point $\bx$ is:
\vspace{-3pt}
\begin{align*}
    & f(\bx) \pvert { \beta}, \tau, \theta, {f}_X \sim \mathcal{N}(\mu_{\theta,\beta}, s_{\theta,\beta}), \\
    &\mu_{\theta, { \beta}} = { \beta}^Tm(\bx) + {K}_{X\bx}^T{K}_{X}^{-1}({f}_X - {M}_X{{ \beta}}),\\
    &s_{\theta, \beta} = \tau^{-1} \big(k_\theta(\bx,  \bx) - {K}^T_{X\bx}{K}^{-1}_{X}{K}_{X\bx}\big),
\end{align*}
where $({K}_{X\bx})_i = k_\theta(\bx_i, \bx)$. Typically, $\tau$, $\beta$, and $\theta$ are fit by minimizing the negative log likelihood:
\vspace{-3pt}
\begin{align*}
-\log \mathcal{L}({f}_X \pvert {X}, { \beta}, \tau, \theta) &\propto \\ 
\frac{1}{2}\big\lVert {f}_X - {M}_X & { \beta} \big\rVert^2_{{K}_{X}^{-1}} + \frac{1}{2}\log \lvert {K}_{X} \rvert.
\end{align*}
This is known as maximum likelihood estimation (MLE) of the kernel hyperparameters. In order to improve the clarity of later sections, we modified the standard GP treatment of \cite{GPbook}; notational differences aside, our formulations are equivalent.

\subsection{Warped Gaussian Processes}

While GPs are powerful tools for modeling nonlinear functions, they make the fairly strong assumption of Gaussianity and homoscedasticity. WGPs \citep{WGP} address this problem by warping the observation space to a latent space, which itself is modeled by a GP. Given a strictly increasing, differentiable parametric transformation $g_\lambda$, WGPs model the composite function $g_\lambda\circ f$ with a GP:
\[(g_\lambda\circ f)\pvert { \beta}, \tau, \lambda, \theta \sim \mathcal{GP}(\mu_{ \beta}, \tau^{-1} k_\theta).\]

Let $(g_\lambda({f}_X))_i = g_\lambda(f(\bx_i))$. WGP jointly computes the parameters through MLE in the latent space, where the negative log likelihood is:
\begin{align*}
-\log \mathcal{L}\big(g_\lambda({f}_X) \pvert {X}, { \beta}, \tau, \theta, \lambda\big) &\propto \\ 
\frac{1}{2}\big\| g_\lambda({f}_X) - {M}_X { \beta} \big\|^2_{{K}_{X}^{-1}} &+ \frac{1}{2}\log \lvert {K}_{X} \rvert - \log J_\lambda,
\end{align*}
and where $J$ represents the transformation Jacobian:
\[ 
J_\lambda = \left\lvert \prod_{i=1}^n \frac{\partial}{\partial f(\bx_i)} g_\lambda(f(\bx_i))\right\rvert. 
\]
WGPs predict the value of a point $ x$ by computing its posterior mean in the latent space and then inverting the transformation back to the observation space: $g_\lambda^{-1}(\hat{\mu}( x))$.
\citet{WGP} uses the tanh transform family, whose members do not generally have closed form inverses; they must be computed numerically.

\subsection{Bayesian Transformed GPs (BTG) \label{ssec:btg}}
\begin{figure}[!t]
    \centering
    \begin{subfigure}[IntSine]{0.45\textwidth}
        \includegraphics[width=\textwidth]{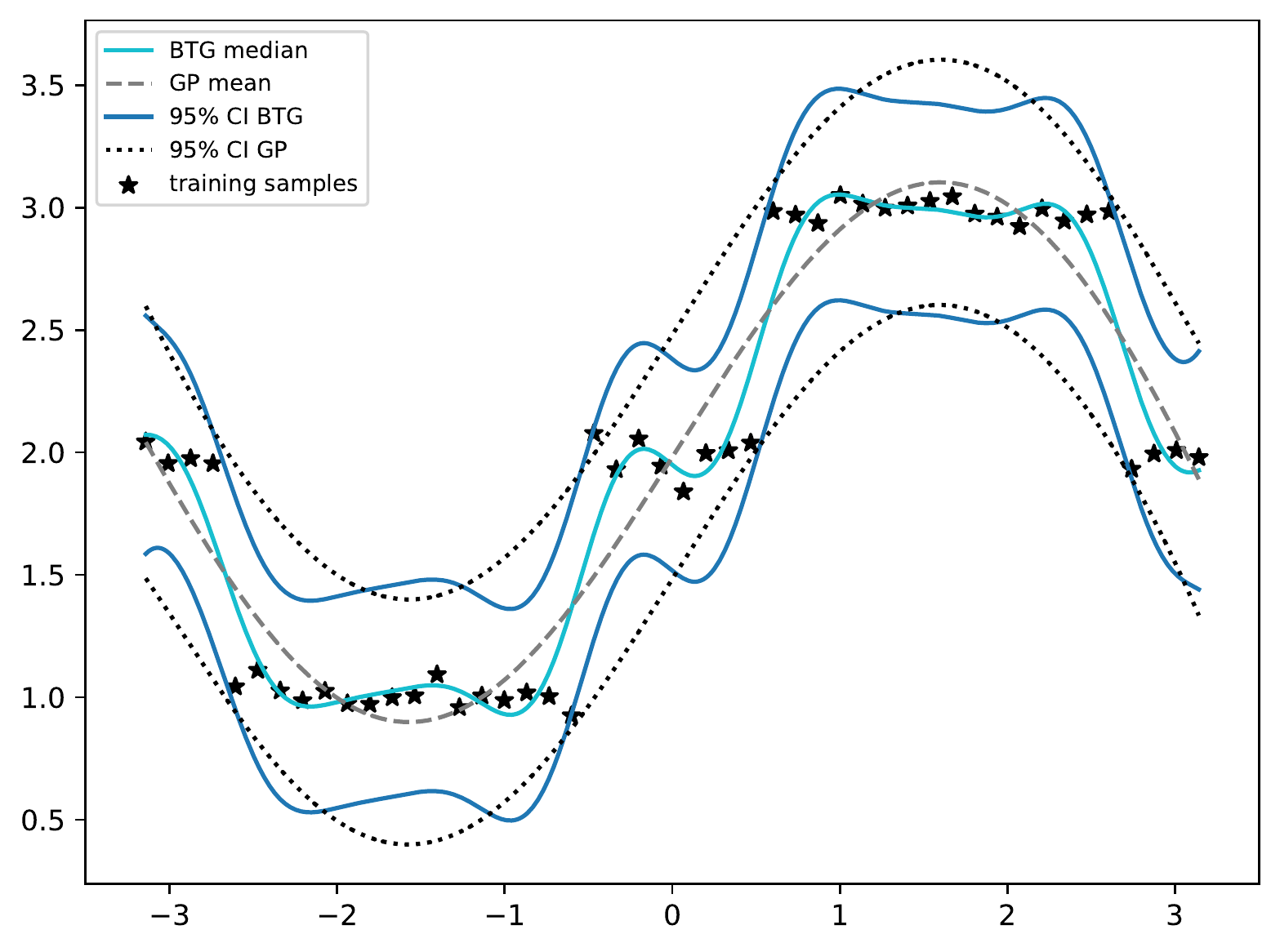}
    \end{subfigure}
    \vspace{-10pt}
    \caption{\footnotesize{A comparison of GP and BTG: predictive mean/median and 95\% equal tailed credible interval. Trained on 48 random samples from the rounded sine function with noise from $\mathcal{N}(0, 0.05)$.}\label{figure: compare btg and gp}}
    \vspace{-10pt}
\end{figure}
One might think of the Bayesian Transformed Gaussian (BTG) model \citep{BTG} as a fully Bayesian generalization of WGP. BTG uses Bayesian model selection and marginalizes out priors over \textit{all} model parameters: transformation parameters $\lambda$, mean vector $\beta$,  signal variance $\tau$, and lengthscales $\theta$. 
Just like WGP, BTG models a function $f(\bx)$ as:
\[(g_\lambda\circ f)\pvert { \beta}, \tau, \lambda, \theta \sim \mathcal{GP}(\mu_\beta, \tau^{-1} k_\theta). \]
BTG was originally a Bayesian generalization of trans-kriging models. Because appropriate values for $\beta$, $\tau$, and $\theta$ depend nontrivially on $\lambda$, BTG adopts the improper joint prior:
\[
p({ \beta}, \tau, \theta, \lambda) \propto  p(\theta)p(\lambda) \;/\; (\tau J_\lambda^{p/n}). 
\]
As it turns out, BTG's posterior predictive distribution can be approximated as a mixture of t-distributions: 
\begin{align*}
p(f(\bx) \pvert f_X) = \sum_{i=1}^M w_i p \big(g_{\lambda_i}(f(\bx)) \pvert \theta_i, \lambda_i, f_X\big),
\end{align*}
where here $p$ is the t-distribution pdf. We provide a condensed derivation in \S \ref{sec:preddensitytreatment} and \ref{sec:methods}; for a comprehensive analysis, see \cite{boxcoxanalysis}. This predictive distribution must then be inverted to perform prediction or uncertainty quantification. 

Figure \ref{figure: compare btg and gp} demonstrates the advantage of fully Bayesian model selection. BTG resolves the underlying datapoints much better than a GP. In later sections, we explore the advantages of being Bayesian in the low-data regime. 

\subsection{The Predictive Density}\label{sec:preddensitytreatment}
A key idea of the BTG model is that, conditioned on $\lambda$, $\theta$, and ${f}_X$, the resulting WGP is a generalized linear model \citep{BTG}. We estimate ${ \beta}$ by $\hat{ \beta}_{\theta,\lambda}$, the solution to the weighted least squares problem:
\begin{equation*}
q_{\theta,\lambda} = \min_\beta \big\| g_\lambda({f}_X) - {M}_X{ \beta}\big\|^2_{{K}^{-1}_{X}}, 
\end{equation*}
where $q_{\theta, \lambda}$ is the residual norm. BTG then adopts a conditional normal-inverse-gamma posterior on $(\beta, \tau)$:
\begin{equation*}
\begin{split}
    { \beta} \pvert \tau, \lambda, \theta, {f}_X &\sim \mathcal{N}\big(\hat\beta_{\lambda,\theta}, \tau^{-1} ({M}_X^T{K}_{ X}^{-1}{M}_X)^{-1}\big), \\
    \tau \pvert \lambda,\theta, {f}_X &\sim \text{Ga}\Big(\frac{n - p}{2}, \frac{2}{q_{\lambda,\theta}}\Big).
\end{split}
\end{equation*}
At a point $\bx$, the marginal predictive density of $g_\lambda(f(\bx))$ is then given by the following $t$-distribution:
\begin{align}
\label{eqn:tdist}
g_\lambda(f(\bx)) \pvert \lambda, \theta, {f}_X \sim T_{n - p}\big(m_{\lambda, \theta}, (q_{\theta,\lambda}{C}_{\theta,\lambda})^{-1}\big), \end{align}
where the mean largely resembles that of a GP:
\begin{equation*}
m_{\lambda,\theta} =   
{K}_{\bx X}{K}_{X}^{-1}\big(g_\lambda({f}_X) - {M}_X\hat\beta_{\lambda,\theta}\big) + \hat\beta_{\lambda,\theta}^Tm(\bx),
\end{equation*} 
and
${C}_{\lambda,\theta}$ is the final Schur complement $B(\bx)/[k_\theta(\bx, \bx)]$ of the bordered matrix: 
\begin{align*} 
B(\bx) = \begin{bmatrix}
 0 & {M}_X^T & {m}(\bx)\\
 {M}_X & {K}_{X} & {K}_{X\bx}\\
 {m}(\bx)^T & {K}_{ X\bx}^T & k_\theta(\bx, \bx)
\end{bmatrix}.
\end{align*}

By Bayes' theorem, the marginal posterior of BTG is: 
\vspace{-5pt}
\begin{equation}\label{eq:BTG marginal posterior}
\begin{split}
&p(f(\bx) \pvert {f}_X) = \\
&\frac{\int_{\Theta, \Lambda} p(g_\lambda(f(\bx)) \pvert \theta, \lambda, {f}_X)p({f}_X \pvert \lambda, \theta )p(\theta)p(\lambda)d\lambda d\theta}{\int_{\Theta, \Lambda} p(f_X |\theta, \lambda)p(\theta)p(\lambda)\,d\lambda\,d\theta}.
\vspace{-15pt}
\end{split}
\end{equation}
Unlike WGP, BTG may not have first or second moments, because its marginal posterior may be for example, a mixture of log-t distributions. If this occurs, the probability density function (pdf) will not have a mean or variance. Therefore, BTG instead uses the median and credible intervals, computed by inverting its cumulative distribution function (cdf).

\section{Methodology}\label{sec:methods}
For general nonlinear transformations, the posterior distribution of BTG (Equation \ref{eq:BTG marginal posterior}) is intractable and therefore we approximate it using a set of $M$ quadrature nodes and weights $([\theta_i, \lambda_i], w_i)$, yielding the mixture of distributions 
\begin{align*}
\begin{split}
&p\big(f(\mathbf x) \pvert f_X \big) \approx\\  &\frac{\sum_{i=1}^M w_i p\big(g_{\lambda_i}(f(\bx)) \pvert \theta_i, \lambda_i, f_X \big)p(f_X | \theta_i, \lambda_i )p(\theta_i)p(\lambda_i)}{\sum_{i=1}^M w_i p(f_X | \theta_i, \lambda_i)p(\theta_i)p(\lambda_i)}.
\end{split}
\end{align*}

In this equation, $p\left(g_{\lambda_i}(f(\bx)) \pvert \theta_i, \lambda_i, f_X \right)$ is the pdf of the $t$-distribution $T_{n-p}(\mu_{\theta_i, \lambda_i}, (q_{\theta_i,\lambda_i}C_{\theta_i,\lambda_i})^{-1})$, $p(f_X | \theta_i, \lambda_i)$ is the likelihood of data given hyperparameters, $p(\theta_i)$ and $p(\lambda_i)$ are our hyperparameter priors, and  $w_i$ is a quadrature weight. 

To simplify notation, we combine all terms except for the t-distribution pdf into the weights $\{\tilde{w}_i\}_{i=1}^M$, where 
\[ \tilde{w}_i := \frac{w_i p(f_X|\theta_i, \lambda_i)p(\theta_i)p(\lambda_i)}{\sum_{i=1}^M w_i p(f_X|\theta_i, \lambda_i) p(\theta_i)p(\lambda_i)}.
\]
Combining constants simplifies the BTG predictive distribution into a general mixture of t-distributions:\vspace{-5pt}
\begin{align}
\label{eq: BTG t-mixture posterior}
p\big(f(\bx) \pvert f_X\big) = \sum_{i=1}^M  \tilde{w}_i p \big(g_{\lambda_i}(f(\bx)) \pvert \theta_i, \lambda_i, f_X\big).
\vspace{-5pt}
\end{align}
As mentioned earlier,  $p\big(f(\bx) \pvert f_X\big)$ is not guaranteed to have a mean, so we must use the median predictor instead. We do so by computing the quantile $P^{-1}(0.5)$ by numerical root-finding, where $P$ is the cdf of $p\big(f(\bx) \pvert f_X\big)$, and therefore a mixture of t-distribution cdfs. 

BTG regression via the median predictor (or any other quantile) of Equation \ref{eq: BTG t-mixture posterior} is challenging. The dimensionality of the integral scales with hyperparameter dimension, which grows not only with the ambient dimension of the data, but also with the number of transformations used. Furthermore, its cdf must be numerically inverted, requiring many such quadrature computations for a single, point-wise regression task. This is further complicated by the difficulty in assessing model fit through LOOCV, which must be repeated at every quadrature node as well. As a result, a naive implementation of BTG scales poorly.

In this section, we discuss scalable algorithms that make BTG model prediction and model validation far faster, and indeed, comparable to the speed of its MLE counterparts. 
First, we discuss our doubly sparse quadrature rules for computing the BTG predictive distribution (\S \ref{sec:sparsegridquad} and \S \ref{sec:quadsparse}). We then provide quantile bounds that accelerate root-finding convergence (\S \ref{sec:quantilbounds}). Next, we propose a $\mathcal{O}(n^3)$ LOOCV algorithm for BTG using Cholesky downdates and rank-1 matrix algebra (\S \ref{sec:fastloocv}). Finally we discuss the single and multi-layer nonlinear transformations used in our experiments (\S \ref{sec:transformations}). 
\subsection{Sparse Grid Quadrature}\label{sec:sparsegridquad}

Sparse grid methods, or Smolyak algorithms, are effective for approximating integrals of sufficient regularity in moderate to high dimensions. While the conventional Monte Carlo (MC) quadrature approach used by \citet{BTG} converges at the rate of $\mathcal{O}(1/\sqrt{M})$, where $M$ is the number of quadrature nodes, the approximation error of Smolyak's quadrature rule is $\mathcal{O}\Big(M^{-r}\lvert \log_2 M\rvert^{(d - 1)(r + 1)}\Big)$ where $d$ is the dimensionality of the integral and $r$ is the integrand's regularity i.e., number of derivatives. 

In this paper, we use a sparse grid rule detailed by \citet{bungartz2004sparse} and used for likelihood approximation by \citet{sparsegrid}.

\subsection{Quadrature Sparsification}\label{sec:quadsparse}
Numerical integration schemes such as sparse grids and QMC use $M$ fixed quadrature nodes, where $M$ depends on the dimensionality of the domain and fineness of the grid. In the Bayesian approach, expensive GP operations such as computing a log determinant and solving a linear system are repeated across quadrature nodes, for a total time complexity of $\mathcal{O}(Mn^3)$. 

In practice, many nodes are associated with negligible mixture weights, so their contribution to the posterior predictive distribution can effectively be ignored. We thus adaptively drop nodes when their associated weights fall below a certain threshold. To do so in a principled way, we approximate the mixture with a subset of dominant weights and then quantify the error in terms of the total mass of discarded weights.

We assume the posterior cdf is the  mixture of cdfs $F(x) = \sum_{i=1}^M w_if_i(x)$, where each $f_i$ is a cdf. Assume the weights $\{w_i\}_{i=1}^M$ are ordered by decreasing magnitude. Consider $F_k(x)$, a truncated and re-scaled $F(x)$.

We first quantify the pointwise approximation error in Lemma \ref{lemma: pointwise error bound}. We then quantify the error in quantile computation: Propositions \ref{qb:pos}, \ref{qb:neg} show that the approximated quantile $F_k^{-1}(p)$ can be bounded by perturbed true quantiles. Proposition \ref{prop: error bound at p} gives a simple bound between $F_k^{-1}(p)$ and $F^{-1}(p)$ within the region of interest, and applies to both QMC---which uses positive weights---and sparse grid quadrature---which uses positive \textit{and} negative weights.

\begin{lemma}
\label{lemma: pointwise error bound}
Let $k$ be the smallest integer such that $\sum_{i=1}^k w_i \geq 1-\epsilon$. Then define the scaled, truncated mixture
 \vspace{-5pt}
\[F_k(x) := \frac{1}{c}\sum_{i=1}^k w_if_i(x),\quad  c := \sum_{i=1}^k w_i.\vspace{-6pt}\]
We have
\[|F(x)-F_k(x)| \le 2\epsilon.\]
\end{lemma}

\begin{proposition}[Error Bound for Positive Weights]\label{qb:pos}

For any $\epsilon \in (0,1)$, let $k$ be the smallest integer such that $\sum_{i=1}^k w_i \geq 1-\epsilon$. Define the scaled, truncated mixture
\vspace{-10pt}
\[
F_k(x) := \frac{1}{c}\sum_{i=1}^k w_if_i(x),\quad  c := \sum_{i=1}^k w_i.
\vspace{-2pt}
\] 
Let $p\in (0, 1)$ and assume that $p\pm 2\epsilon\in (0, 1)$. Then the approximate quantile $ F_k^{-1}(p)$ is bounded by perturbed true quantiles:
\[F^{-1}(p-2\epsilon) \le F_k^{-1}(p) \le F^{-1}(p+2\epsilon).\]
\end{proposition}

\begin{proposition}[Error Bound for Negative Weights]\label{qb:neg}
Let $F(x)$ be defined as before, except each $w_i$ is no longer required to be positive. Consider the split $F(x) = F_{M'}(x)+R_{M'}(x)$, where $F_{M'}(x) = \sum_{i=1}^{M'}w_if_i(x)$ and $R_{M'}(x) = \sum_{i=M'+1}^M w_if_i(x)$. Then for any $x$, we have $R_{M'}(x)\in [\epsilon_{-}, \epsilon_+]$, where the epsilons are defined as the sum of positive (resp. negative) weights of $R_{M'}(x)$
\[
\epsilon_{-} = \sum_{i = M'+1}^M [w_i]_{-} \le 0 \;,\; \epsilon_{+} = \sum_{i = M'+1}^M [w_i]_{+} \ge 0.
\]
Let $p\in (0, 1)$ and assume $p + \epsilon_{-}, p + \epsilon_{+}\in (0, 1)$. Then the approximate quantile $ F_{M'}^{-1}(p)$ is bounded by perturbed true quantiles:
\[ 
F^{-1}(p + \epsilon_-) \le F_{M'}^{-1}(p) \le F^{-1}(p + \epsilon_{+}).  
\]
\end{proposition}

\begin{proposition}[Error Bound at a quantile\label{prop: error bound at p}]
Let $F(x)$ be defined as before, $\epsilon_1$, $\epsilon_2 \in (0,1)$, and $F_k(x)$ be an approximate to $F(x)$ such that $F^{-1}(p-\epsilon_1) \leq F_k^{-1}(p) \leq F^{-1}(p-\epsilon_2)$ for some $p\in(0,1)$. Assuming $p-\epsilon_1, p+\epsilon_2 \in (0,1)$, we have the following error bound at a quantile,\vspace{-5pt}
$$\left|F_k^{-1}(p) - F^{-1}(p)\right| \leq \epsilon \max_{\xi\in(p-\epsilon_1, p+\epsilon_2)} \left|\frac{d F^{-1}}{d x}(\xi)\right| ,$$
where $\epsilon = \max\{\epsilon_1, \epsilon_2\}$.
\end{proposition}

\vspace{-2pt}
By adaptively sparsifying our numerical quadrature schemes, we are able to discard a significant portion of summands in the mixture $F(x)$, which in turn, enables significant speedup of BTG model prediction. Empirical results are shown in $\S\ref{sec:experiments}$.

\subsection{Quantile Bounds}\label{sec:quantilbounds}
To compute posterior quantiles, we apply Brent's algorithm, a standard root-finding algorithm combining the secant and bisection methods, to the cdf defined in Equation \ref{eq: BTG t-mixture posterior}. Since Brent's algorithm is box-constrained, we use quantile bounds to narrow down the locations of the quantiles for $p\in \{0.025, 0.5, 0.975\}$. 

Let $F(x) = \sum_{i=1}^M w_if_i(x)$. Then we have the following bounds for the quantile $F^{-1}(p)$.

\begin{proposition}[Convex Hull]\label{qb:cv}
Let $F(x)$ be defined as before with $w_i>0$ and $\sum_{i=1}^M w_i = 1$. Then
\[ \min_i f_i^{-1}(p) \le F^{-1}(p) \le \max_i f_i^{-1}(p).\]
\end{proposition}

\vspace{-2pt}
\begin{proposition}[Singular Weight]\label{qb:sing} Let $F(x)$ be defined as before with $w_i>0$ and $\sum_{i=1}^M w_i = 1$. Let $\overline{w}_i = 1-w_i$. Then
\begin{align*}
\max_{p - \overline{w}_i \ge 0} f_i^{-1}(p- \overline{w}_i ) \le F^{-1}(p) 
\le \min_{p+\overline{w}_i  \le 1} f_i^{-1}(p+\overline{w}_i ).
\end{align*}
\end{proposition}
\vspace{-2pt}
When solving for $y^* = F^{-1}(p)$, we run Brent's algorithm using our quantile bounds as the box constraints. Furthermore, we adaptively set the termination conditions xtol and ftol to be on the same order of magnitude as the error in quadrature sparsification from \S \ref{sec:quadsparse}. This greatly accelerates convergence in practice. A comparison between the performances of the two quantile bounds outlined in this section can be found in \S \ref{sec:experiments}.

\subsection{Fast Cross Validation}\label{sec:fastloocv}
LOOCV is standard measure of model fit: in practice, it is most commonly used for tuning hyperparameters and for model selection. While fast LOOCV schemes are known for GP regression, it is less straightforward to perform LOOCV on BTG. In particular, the computational difficulty lies in two LOOCV sub-problems: a generalized least squares problem and principle sub-matrix determinant computation. These correspond to the terms in the BTG likelihood function and the BTG conditional posterior in Equation \ref{eq:BTG marginal posterior}. Being Bayesian about covariance and transform hyperparameters introduces additional layers of cost: LOOCV must be repeated at each quadrature node in hyperparameter space. This further motivates the need for an efficient algorithm.

For notational clarity, let $(-i)$ denote the omission of the $i$th point. For a kernel matrix, this means deletion of the $i$th row and column; for a vector, this indicates the omission of the $i$th entry. We seek to compute the mean $m_{\theta, \lambda}^{(-i)}$ and standard deviation $\sigma^{(-i)}_{\theta, \lambda} = \big({C}_{\theta, \lambda}^{(-i)}q_{\theta, \lambda}^{(-i)}\big)^{-1/2}$ of the t-distributions (Equations \ref{eqn:tdist}) for each submodel, obtained by leaving out the $i$th training point. Specifically, computing $\{q_{\theta, \lambda}^{(-i)}\}_{i=1}^n$ entails solving the generalized least squares problems for $i = 1, \dots, n$:
\begin{equation*}
\vspace{-5pt}
\argmin_{{ \beta}^{(-i)}} \big\|{Y}^{(-i)} -  {{M}_X}^{(-i)}{\beta}^{(-i)}\big\|_{{ K}_X^{(-i)}}^2,
\end{equation*}
\vspace{-5pt}where ${Y} = g\circ {f}_X$. In addition, computing ${C}_{\theta, \lambda}^{(-i)}$ and $m_{\theta, \lambda}^{(-i)}$ entails solves with ${K}_X^{(-i)}$, which naively takes $\mathcal{O}(n^3)$ per sub-problem. Therefore, the BTG LOOCV proceedure naively takes $\mathcal{O}(n^4)$ total time.

We develop an $\mathcal{O}(n^3)$ fast LOOCV algorithm for BTG using three building blocks: fast determinant computations (Proposition \ref{prop: det of principal minor}), fast abridged linear system solves (Proposition \ref{prop: Abridged LS}) and fast rank-one $\mathcal{O}(p^2)$ Cholesky down-dates (Proposition \ref{prop: downdate for matrices}). We refer to \cite{stewartvol1} for the rank-1 Cholesky downdate algorithm. For algorithm details as well as proofs, we refer to the supplement. The scaling behavior for our LOOCV algorithm is shown in Figure \ref{fig:loocv} in \S \ref{sec:experiments}.

\begin{proposition}[Determinant of a Principal Minor\label{prop: det of principal minor}]
\vspace{-2pt}
\begin{equation*}
\det\big(\Sigma^{(-i)}\big) = \det(\Sigma)\big(e_i^T\Sigma^{-1}e_i\big).
\end{equation*}
\end{proposition}

\begin{proposition}[Abridged Linear System\label{prop: Abridged LS}]
Let $K\in \mathbb{R}^{n\times n}$ be of full rank, and let ${c}, {y}\in\mathbb{R}^n$ satisfy ${K}{c} = {y}$. Then if $r_i = c_i / {e}_i^{T}{K}^{-1}{e}_i$, we have:
\[{c}^{(-i)} = ({K}^{(-i)})^{-1}{y}^{(-i)} = {c} - r_i{K}^{-1}{e}_i .\]
\end{proposition}

\begin{proposition}[Rank one matrix downdate\label{prop: downdate for matrices}]
If $X\in\mathbb{R}^{n\times m}$ with $m<n$ has full column rank and $\Sigma$ is a positive definite matrix in $\mathbb{R}^{n\times n}$, then we have
\begin{equation*}
{X}^{(-i)^T} {\Sigma}^{(-i)^{-1}}{X}^{(-i)} = {X}^T \left( {\Sigma}^{-1} - \frac{{\Sigma}^{-1}{e}_i{e}_i^T{\Sigma}^{-1}}{{e}_i^T{\Sigma}^{-1}{e}_i} \right){X},
\end{equation*}
where ${\Sigma}^{(-i)}\in\mathbb{R}^{(n-1)\times(n-1)}$ is the $(i, i)th$ minor of ${\Sigma}$ and ${e}_i$ is the $i$th canonical basis vector.
\end{proposition}

\vspace{-5pt}
\subsection{Transformations}\label{sec:transformations}
\begin{table}[tbp]
\small
\centering
\caption{Elementary Transformations: analytic function forms and parameter constraints. Parameters are assumed to be in $\mathbb{R}$ unless stated otherwise. \label{table: transforms}}
\begin{tabular}{cccc}
\textbf{Name}    & $g(y)$ & \textbf{Req.} & \textbf{Count} \\ \hline 
Affine      & $a + by$     &   $b > 0$    & 1     \\
ArcSinh     & $a + b \text{~asinh}\left(\dfrac{y-c}{d}\right)$  & $b, d>0$ & 4 \\
SinhArcSinh &   $\text{sinh}(b \text{~asinh(y} - a))$   &  $b > 0$ & 2\\
Box-Cox     &  
$\begin{cases} \dfrac{y^\lambda - 1}{\lambda} & \mbox{if } \lambda > 0 \\ 
\log(y) & \mbox{if } \lambda = 0 \end{cases}$
  &   $\lambda \geq 0$ & 1 \\\hline 
\end{tabular}
\end{table}

The original BTG model of \cite{BTG} uses the Box-Cox family of power transformations and places an uniform prior on $\lambda$. Recent research has greatly expanded the set of flexible transformations available. \cite{WGP} uses a sum of tanh$(\cdot)$ transforms in the WGP model and \cite{CWGP} composes various transformations to provide a flexible compositional framework in the CWGP model. 

We apply BTG with more elementary transformations and compositions thereof, summarized in Table \ref{table: transforms}. As we show in \S \ref{sec:experiments}, these compositions have greater expressive power and generally outperform single transformations, at the expense of greater computational overhead. 
\vspace{-2pt}

\begin{figure*}[!t]   
    \centering
    \vspace{.3in}
        \centering
        \includegraphics[width=\textwidth]{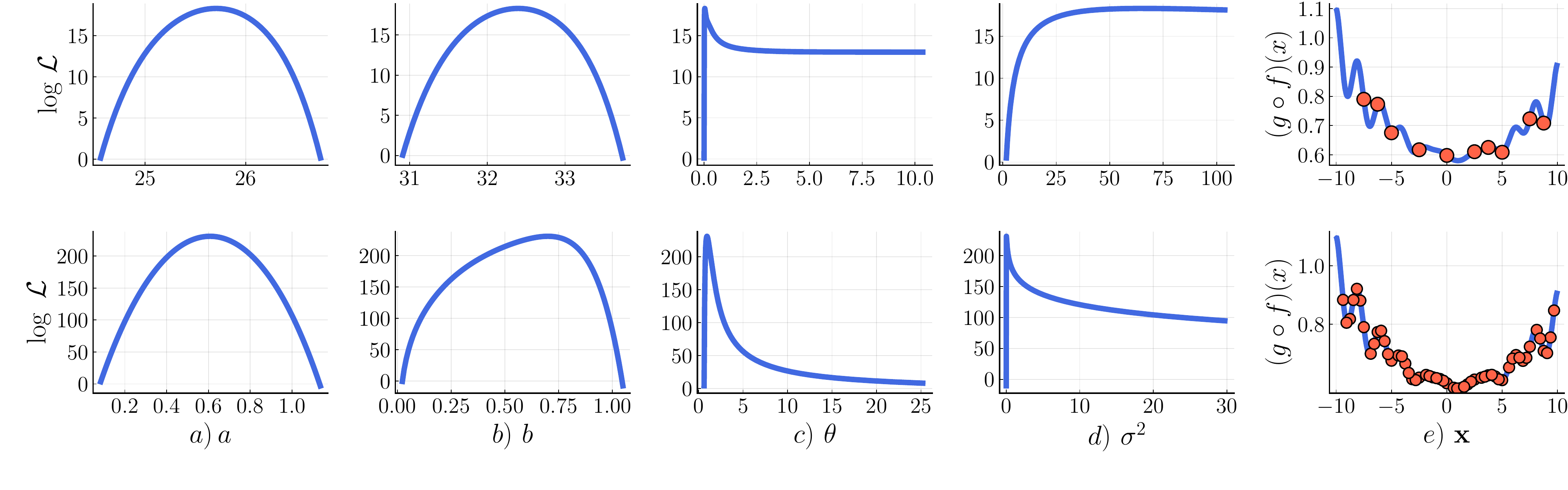}
    \vspace{-25pt}
    \caption{Plots of marginal log likelihoods for a SinhArcSinh (SA)-transformed WGP. In the first four columns, $a, b$ are transformation hyperparameters, and $\theta, \sigma^2$ are kernel hyperparameters. The top row represents a data-sparse 
    setting, while the bottom row represents a data-rich setting. $\sigma^2$ is only well-defined in the latter setting.}
    \label{fig: likelihood sensitivity in data-sparse}
    \vspace{-10pt}
\end{figure*}
\section{Experiments}\label{sec:experiments}
\vspace{-3pt}
We first perform a set of scaling experiments to validate the efficiency of our algorithms. Our efficient computational techniques enable us to run a series of thorough regression experiments, which demonstrate BTG's strong empirical performance when compared to appropriately selected baselines. 
\vspace{-2pt}
\subsection{Motivation for the Bayesian Approach}
\vspace{-3pt}
We examine the marginal log likelihoods of transformation and kernel parameters in the WGP model in Figure \ref{fig: likelihood sensitivity in data-sparse}. We observe that in data-sparse settings, the likelihood of $\theta$ and $\sigma^2$ are poorly-defined, with many possible hyperparameters that explain the data, while in data-rich settings, the distribution of $\theta$ and $\sigma^2$ are tightly concentrated. This suggests that being Bayesian about hyperparameters---which is the approach that BTG takes---could be more appropriate in the data-sparse setting than MLE estimation.

\subsection{Scaling experiments}

\begin{figure}[!hbp]
\label{fig:loocv}
    \vspace{-10pt}
    \centering
    \begin{subfigure}[IntSine]{0.5\textwidth}
        \centering
        \includegraphics[width=\textwidth]{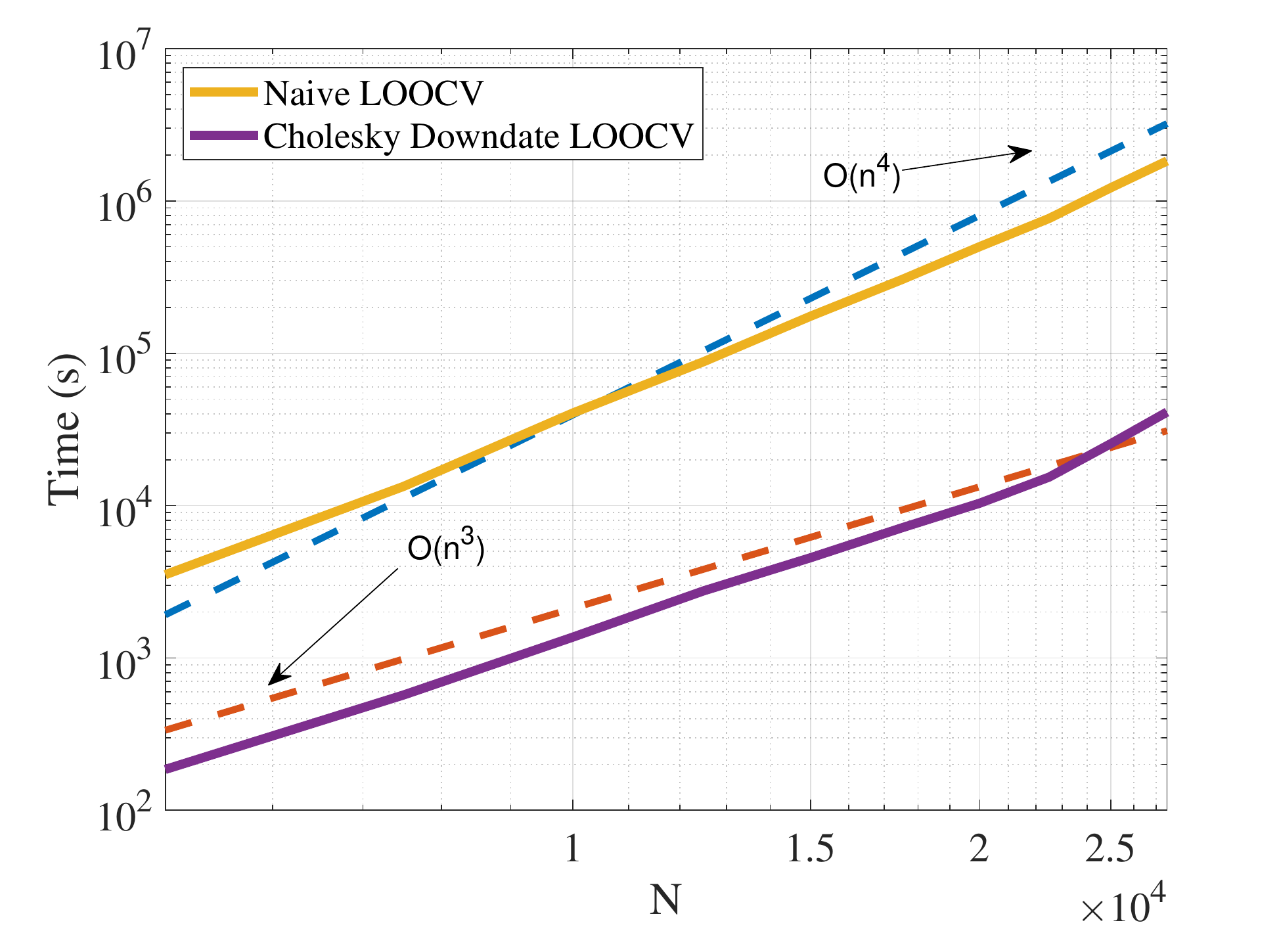}
    \end{subfigure}
    \vspace{-13pt}
    \caption{\small{LOOCV timing: computing posterior distribution parameters of $N$ sub-models. 
    }}
    \vspace{-10pt}
    \label{fig:loocv}
\end{figure}

\textbf{Fast Cross Validation}
To assess our fast LOOCV, we infer the sub-model posterior distribution moments on a toy problem in two different ways: with and without using Cholesky rank-one downdates on $R_X$ to compute the Cholesky factors for $R_X^{(-i)}$. We plot timing results in Figure \ref{fig:loocv}, which confirm that our $\mathcal{O}(n^3)$ method scales significantly better than the naive $\mathcal{O}(n^4)$ method.

\textbf{Sparse Grids vs QMC} We compare sparse grid and QMC quadrature rules under our quadrature sparsification framework. In our experiment, we begin with a handful of quadrature nodes, and gradually extend this set to the entire quadrature grid. 

We plot the resulting prediction errors in Figure \ref{fig:quad_compare}, and observe that the sparse grid quadrature rule yields lower prediction error: QMC converges to an MSE of $3.99$, while sparse grid converges to an MSE of $3.80$. We also observe that the sub-grids have similar \textit{weight concentration}, which we showed was a proxy for quadrature approximation error in \S \ref{sec:methods}. Therefore, as the mass of the dropped weights falls below $0.1$, the error in the integration scheme can increasingly be attributed to the error in the quadrature rule itself as opposed to sparsification. Since the error of the sparse grid rule decays faster than that of QMC, we expect sparse grid prediction error to also decay more quickly. This trend is supported by Figure $\ref{fig:quad_compare}$.

\begin{figure*}[!t]
    \centering
    \includegraphics[width = \textwidth]{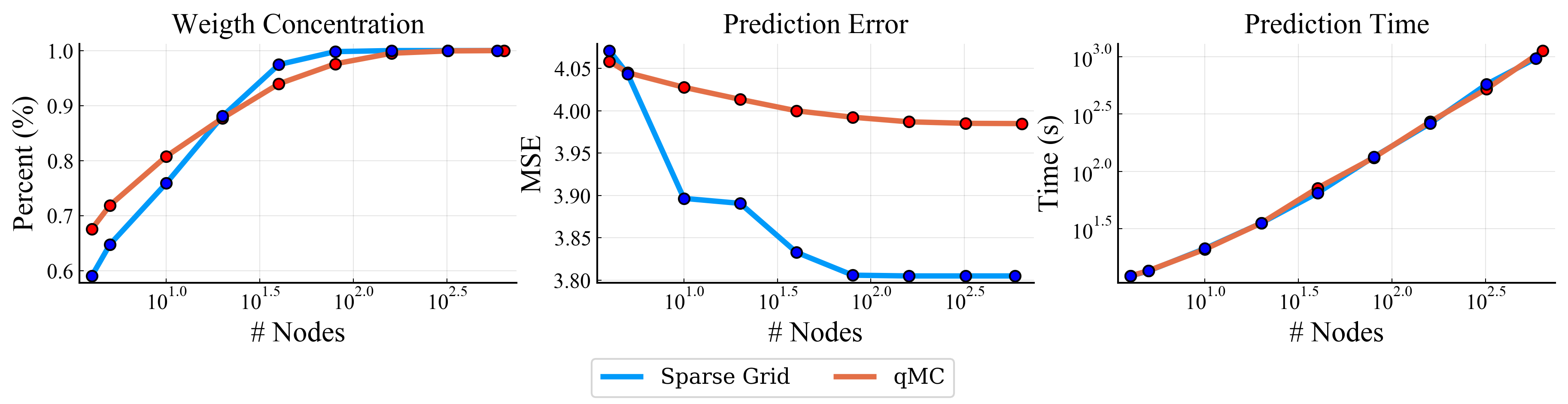}
    \vspace{-15pt}
    \caption{Comparison of sparse grid and QMC. Weights are ordered by decreasing magnitude. (\textit{Left}) Number of nodes versus percent of total mass they comprise  (\textit{Middle}) Number of nodes versus BTG prediction MSE (\textit{Right}) Number of nodes versus prediction time.}
    \label{fig:quad_compare}
\end{figure*}

We find that the joint-likelihood function is sufficiently smooth, so that sparse grids are effective. Finally, we confirm that inference time scales linearly with the number of quadrature nodes.

\textbf{Quantile Bounds Speedup} To assess the effectiveness of quantile bounds for root-finding, we record BTG prediction times using the convex hull bound and singular weight bound. We find that the convex hull bound decreases the overall computational overhead by a factor of at least two. The convex hull bound outperforms the singular weight bound for finding credible intervals and in overall time, but the singular weight bound was faster for finding the median in many scenarios. A detailed table of results can be found in the supplement.

\vspace{-10pt}
\subsection{Regression Experiments}

\begin{table*}[h] 
\centering 
\caption{Performance Results on IntSine, SixHumpCamel, Abalone, Wine and Creep datasets. Models: GP, WGP, CWGP, BTG. Transformations: I:identity, A:ArcSinh, SA:SinhArcSinh, BC:BoxCox, L:affine. RMSE and MAE are shown as metrics, the lower the better and best ones are bold.}
\begin{tabular}{lllllllllll}
\hline
          & \multicolumn{2}{c}{\textbf{IntSine}} & \multicolumn{2}{c}{\textbf{Camel}} & \multicolumn{2}{c}{\textbf{Abalone}} & \multicolumn{2}{c}{\textbf{Wine}} & \multicolumn{2}{c}{\textbf{Creep}} \\ \hline
\textbf{} & RMSE & MAE   & RMSE & MAE & RMSE   & MAE   & RMSE     & MAE     & RMSE  & MAE  \\ \hline
GP        & 0.227  & 0.171 & 2.003& 1.781  & 3.290  & 2.208 & 1.994    & 1.792   & 37.88 & 25.68\\
WGP-A     & 0.179  & 0.117 & 2.055& 1.815  & 3.097  & 2.068 & 0.811    & 0.677   & 35.48 & \textbf{24.27}  \\
WGP-SA    & 0.172  & 0.109 & 2.012& 1.788  & 2.992  & 2.022 & 0.809    & 0.670   & 61.62 & 45.85\\
WGP-BC    & 0.184  & 0.123 & 1.964& 1.751  & 2.826  & 1.940 & 1.045    & 0.784   & 40.38 & 25.49\\
CWGP-L-SA & 0.172  & 0.109 & 2.055& 1.815  & 3.117  & 2.100 & 0.808    & 0.670   & 91.89 & 73.05\\
CWGP-A-BC & 0.174  & 0.113 & 1.960& 1.747  & 3.088  & 2.069 & 0.808    & 0.670   & 37.89 & 25.34\\
BTG-I     & 0.169  & 0.100 & 1.820& 1.731  & 2.804  & 1.842 & 0.808    & 0.670   & 38.11 & 26.18\\
BTG-A     & 0.168  & 0.101 & 1.827& 1.741  & 2.890  & 1.900 & 0.807    & 0.669   & 38.68 & 27.68\\
BTG-SA    & 0.165  & 0.104 & 1.801& 1.796  & \textbf{2.791}    & 1.822 & 0.820    & 0.696   & 91.69 & 74.04\\
BTG-BC    & 0.170  & 0.102 & 1.675& 1.666  & 3.225  & 2.172 & 0.808    & 0.670   & 39.10 & 26.52\\
BTG-L-SA  & \textbf{0.145}    & \textbf{0.082}   & \textbf{1.673}  & \textbf{1.658} & 2.871  & 1.870 & 0.809    & 0.670   & 91.83 & 73.05\\
BTG-A-BC  & 0.159  & 0.090 & 1.828& 1.742  & 2.832  & \textbf{1.814}   & \textbf{0.802}      & \textbf{0.664}     & \textbf{35.25}   & 24.44  \\
\hline
\end{tabular}
\label{tab:regression}
\end{table*}

Our efficient algorithms allow us to test BTG on a set of synthetic and real-world regression tasks. We first consider two low-dimensional synthetic functions: \textbf{IntSine} and \textbf{SixHumpCamel} \citep{sixhumpcamel} of dimension 1 and 2, respectively. For synthetic functions, we sample training data using a Sobol sequence. We then consider high-dimensional real-world datasets from the UCI repository \citep{Dua:2019}: \textbf{Abalone}, \textbf{Wine}, and \textbf{Creep} of dimensions 8, 12, and 30. The total dimension of the hyperparmeter space is further inflated by transform parameters by as much as $8$. 

We compare with a standard GP model, WGP models and CWGP models with the same set of elementary transformations and their compositions in Table \ref{table: transforms}. We leave the tanh$(\cdot)$ transform in \cite{WGP} out of experiments since it is shown that elementary transforms in CWGP are competitive with the tanh$(\cdot)$ transform in \cite{CWGP}. 

We record regression root mean squared error (RMSE) and mean absolute error (MAE) in Table \ref{tab:regression}. The training set sizes vary from $50$ to $200$ (see supplement for full details). We observe that BTG outperforms other baselines on all the synthetic datasets and a majority of real datasets. Critically, the end-to-end inference time of BTG is comparable to other baselines; see the supplement for timing details. 

We also observe that composed BTG models tend to outperform single-transformation BTG models, which themselves tend to outperform their GP, WGP, and CWGP counterparts. This demonstrates the improved flexibility afforded by layered transformations, and is evidence of the superior performance possible with a fully Bayesian approach on small to medium datasets.

\section{Conclusion}
We have shown that a combination of sparse grid quadrature, quadrature sparsification, and tight quantile bounds significantly reduces the expense of the Bayesian approach---in certain cases rivaling even the speed of MLE---without degrading prediction accuracy. Furthermore, we proposed a fast BTG LOOCV algorithm for model selection and assessing model fit. 
Our framework allows the practitioner to control the trade-off between the speed and accuracy of the Bayesian approach by modulating the sparsification of the grid and tolerance of the quantile-finding routine. 
Lastly, we show that BTG compares favorably to WGP in terms of prediction accuracy on a set of synthetic and real regression experiments.

In future work, we would like to combine our approach with approximate GP inference to further improve computational efficiency. In addition, we would like to apply BTG to Bayesian optimization.

\bibliography{ref}

\newpage
\appendix
\section{Methodology}

\subsection{Quadrature Sparsification}
 We assume the posterior cdf is the mixture of cdfs $F(x) = \sum_{i=1}^M w_if_i(x)$, $0\le f_i(x)\le 1$, and $f_i(x)$ are monotone increasing for $i = 1, \dots, M$. Assume the weights $\{w_i\}_{i=1}^M$ are decreasingly ordered by magnitude from $1$ to $M$. We consider the quantiles of the approximant $F_k(x)$, a truncated and re-scaled $F(x)$.  

\begin{lemma}
 Define $k$ to be the smallest integer such that $\sum_{i=1}^k w_i \geq 1-\epsilon$. Then define the scaled, truncated mixture
 \vspace{-10pt}
\[F_k(x) := \frac{1}{c}\sum_{i=1}^k w_if_i(x),\quad  c := \sum_{i=1}^k w_i\vspace{-6pt}\]
We have
\begin{equation*}
|F(x)-F_k(x)| \le 2\epsilon.
\end{equation*}
\end{lemma}
\begin{proof}
Let $R_k(x) = |F_k(x) - F(x)|$. We have
\begin{align*}
R_k(x)& = \left|\left(\frac{1}{c} - 1\right)\sum_{i=1}^k w_if_i(x) + \sum_{i=k+1}^M w_if_i(x) \right|  \\
&\le \left|\left(\frac{1}{c} - 1\right)\sum_{i=1}^k w_if_i(x)\right| + \left| \sum_{i=k+1}^M w_if_i(x) \right| \\ 
& \leq \frac{1-c}{c} \sum_{i=1}^k w_if_i(x) + 1-c \\
& \leq 2(1-c) \leq 2\epsilon.
\end{align*}
\vspace{-6pt}
\end{proof}

\begin{proposition}[Error Bound for Positive Weights]\label{qb:dom}
For any $\epsilon \in (0,1)$, let $k$ be the smallest integer such that $\sum_{i=1}^k w_i \geq 1-\epsilon$. Then define the scaled, truncated mixture
\vspace{-10pt}
\[
F_k(x) := \frac{1}{c}\sum_{i=1}^k w_if_i(x),\quad  c := \sum_{i=1}^k w_i.
\vspace{-5pt}
\] 
Let $p\in (0, 1)$ and assume that $p\pm 2\epsilon\in (0, 1)$. Then we have the bound:
\begin{equation*}
F^{-1}(p-2\epsilon) \le F_k^{-1}(p) \le F^{-1}(p+2\epsilon).
\end{equation*}
\end{proposition}
\begin{proof}
Let $F_k(x^*)=p$. Then $|p-F_k(x^*)|\le 2\epsilon$, so
\begin{equation*}p-2\epsilon \le F(x^*)\le p+2\epsilon\end{equation*}
It follows that 
\begin{equation*}F^{-1}(p-2\epsilon) \le x^* \le F^{-1}(p+2\epsilon).\end{equation*}
\end{proof}
\begin{proposition}[Error Bound for Negative Weights]\label{qb:neg}
Let $F(x)$ be defined as before, except each $w_i$ is no longer required to be positive. Consider the split $F(x) = F_{M'}(x)+R_{M'}(x)$, where $F_{M'}(x) = \sum_{i=1}^{M'}w_if_i(x)$ and $R_{M'}(x) = \sum_{i=M'+1}^M w_if_i(x)$. Then for any $x$, we have $R_{M'}(x)\in [\epsilon_{-}, \epsilon_+]$, where the epsilons are defined as the sum of positive (resp. negative) weights of $R_{M'}(x)$
\vspace{-5pt}
\[
\epsilon_{-} = \sum_{i = M'+1}^M [w_i]_{-} \le 0 \;,\; \epsilon_{+} = \sum_{i = M'+1}^M [w_i]_{+} \ge 0.
\vspace{-5pt}
\]
Then we bound $F^{-1}(p)$ as follows:
\vspace{-2pt}
\[ 
F^{-1}(p + \epsilon_-) \le F_{M'}^{-1}(p) \le F^{-1}(p + \epsilon_{+}).  
\vspace{-5pt}
\]
\end{proposition}
\begin{proof}
Let $F_{M'}(x^*) = p$. Then we can wrote $F(x^*) = p + R_{M'}(x^*)$. Since $\epsilon_{-} \le R_{M'}(x^*) \le \epsilon_+$, it follows that
\[p + \epsilon_- \le F(x^*) \le p + \epsilon_{+} \]
from which the result follows.
\end{proof}

\begin{proposition}[Error Bound at a quantile\label{prop: error bound at p}]
Let $F(x)$ be defined as before, $\epsilon_1$, $\epsilon_2 \in (0,1)$, and $F_k(x)$ be an approximate to $F(x)$ such that $F^{-1}(p-\epsilon_1) \leq F_k^{-1}(p) \leq F^{-1}(p-\epsilon_2)$ for some $p\in(0,1)$. Assuming $p-\epsilon_1, p+\epsilon_2 \in (0,1)$, we have the following error bound at a quantile,\vspace{-5pt}
$$\left|F_k^{-1}(p) - F^{-1}(p)\right| \leq \epsilon \max_{\xi\in(p-\epsilon_1, p+\epsilon_2)} \left|\frac{d F^{-1}}{d x}(\xi)\right| ,$$
where $\epsilon = \max\{\epsilon_1, \epsilon_2\}$.
\end{proposition}
\begin{proof} We have \begin{align*}
    & \left|F_k^{-1}(p) - F^{-1}(p)\right| \\  \leq &  
      \max \{|F^{-1}(p-\epsilon_1) - F^{-1}(p)|, |F^{-1}(p-\epsilon_2) - F^{-1}(p)| \} \\
      \leq & \max \{\epsilon_1, \epsilon_2 \} \max_{\xi\in(p-\epsilon_1, p+\epsilon_2)} \left|\frac{d F^{-1}}{d x}(\xi)\right|.
\end{align*}
\end{proof}

\subsection{Quantile Bounds}
\begin{proposition}[Convex Hull\label{qb:cv}]
Let $F(x)$ be defined as before with $w_i>0$ and $\sum_{i=1}^M w_i = 1$. Then
\[ \min_i f_i^{-1}(p) \le F^{-1}(p) \le \max_i f_i^{-1}(p).\]
\end{proposition}
\begin{proof}
\vspace{-10pt}
Assume for the sake of contradiction that $F^{-1}(p)> \max_i f_i^{-1}(p)$. Let $k = \argmax_i f_i^{-1}(p)$. Then  
\begin{align*}
    p &> F(f_{i^*}^{-1}(p)) 
    \\&= \sum_{j=1}^M w_jf_j(f_{k}^{-1}(p))
    \\&= w_{k}p + \sum_{j\ne i} w_jf_j(f_{k}^{-1}(p))
    \\& \ge w_kp + 1 - w_k.
\end{align*}
This implies \vspace{-8pt}$$p-1+w_k-w_kp >0 \Longleftrightarrow (w_k-1)(1-p)>0,\vspace{-8pt}$$ which leads to a contradiction. The lower bound is analogous.
\end{proof}

\begin{proposition}[Singular Weight\label{qb:sing}] Let $F(x)$ be defined as before with $w_i>0$ and $\sum_{i=1}^M w_i = 1$. Let $\overline{w}_i = 1-w_i$. Then
\begin{align*}
\max_{p - \overline{w}_i \ge 0} f_i^{-1}(p- \overline{w}_i ) \le F^{-1}(p) 
\le \min_{p+\overline{w}_i  \le 1} f_i^{-1}(p+\overline{w}_i ).
\end{align*}
\end{proposition}
\begin{proof}
Assume for sake of contradiction that
\[F^{-1}(p) > f_i^{-1}(p+\overline{w}_i)\]
Then
\begin{align*}
    p &> F\left(f_i^{-1}(p+\overline{w}_i)\right)
    \\& = \sum_{j=1}^N w_j f_j \left(f_i^{-1}(p+\overline{w}_i)\right) 
    \\& = w_i(p+\overline{w}_i) + \sum_{j\ne i} w_jf_j\left(f_i^{-1}(p+\overline{w}_i)\right)
    \\& \ge w_i(p+\overline{w}_i)
\end{align*}
However this implies that $0 > \overline{w}_i(w_i-p)$, which is a contradiction because $1\ge w_i$ and $w_i \ge p$ (by the assumption that $p+\overline{w}_i\le 1$). The lower bound is analogous.
\end{proof}

\section{FAST CROSS VALIDATION}
\label{fastloocv}
 In this section, we discuss results leading up to $\mathcal{O}(n^3)$ LOOCV algorithms for BTG, which are given by Algorithm \ref{algo: btgloocv} and Algorithm \ref{algo: btgdet}. Naively, the BTG LOOCV procedure has $\mathcal{O}(n^4)$ time cost, due to the costs associated with solving generalized least squares problems related by single-point deletion and evaluating determinants of principle submatrices. We present relevant propositions used to solve these LOOCV sub-problems efficiently in \S\ref{subsections: building blocks}, and derive our full algorithm in \S\ref{subsection:derivation}.
 
\textbf{Notation} Let ${f}_X, {M}_X, {K}_X, \bx$, $\sigma_{\theta, \lambda}, q_{\theta, \lambda}, m_{\theta, \lambda}$ and $C_{\theta, \lambda}$ be defined same as in the paper. As before, we use the $(-i)$ notiation to represent to omission of information from the $i$th data point. For the BTG LOOCV problem, we consider the $n$ submodels $\{\text{Model}^{(-i)}\}_{i=1}^n$ trained on $\{\textbf{x}^{(-i)}, f_X^{(-i)}, M_X^{(-i)}\}_{i=1}^n$: the location-covariate-label triples obtained by omitting data points one at a time. We wish to efficiently compute the posterior predictive distributions of all $n$ submodels indexed by $i\in\{1, ..., n\}$,
\begin{equation}
\label{eq:tmixture}
 p\big(f(\mathbf x^{(-i)}) \pvert f_X^{(-i)} \big) \propto   \sum_{j=1}^M w_j L_j J_j p(\theta_j)p(\lambda_j),
\end{equation}
where $L_j = p\left(g_{\lambda_j}(f(\bx^{(-i)})) \pvert \theta_j, \lambda_j, f_X^{(-i)} \right) $ and \\ $J_j =  p\big(f_X^{(-i)} | \theta_j, \lambda_j\big)$ for $ j \in \{1, 2, \dots, M\}$.

Recall that in Equation \ref{eq:tmixture}, $p\big(g_{\lambda}(f(\bx)) \pvert \theta, \lambda, f_X \big)$ is the probability density function of the $t$-distribution $T_{n-p}(m_{\theta, \lambda}, (q_{\theta,\lambda}C_{\theta,\lambda})^{-1})$ and $p(f_X | \theta, \lambda)$ is the likelihood of data given hyperparameters. 

\textbf{Problem Formulation}
We have to efficiently compute the parameters that of the posterior mixture of t-distributions in Equation \ref{eq:tmixture}: \[\{\text{TParameters}^{(-i)}\}_{i=1}^n := \bigg\{m_{\theta_i, \lambda_i}^{(-i)}, q_{\theta_i, \lambda_i}^{(-i)}, C_{\theta_i, \lambda_i}^{(-i)}\bigg\}_{i=1}^n\]
For definitions of these quantities, we refer to the main text. We instead emphasize here that solving for $q^{(-i)}_{\theta, \lambda}$ entails solving perturbed generalized least squares problems and that solving for $m^{(-i)}_{\theta, \lambda}$ and $C^{(-i)}_{\theta, \lambda}$ entail solving perturbed linear systems.

For the likelihood term in Equation \ref{eq:tmixture}, we have 
 \[
 p(f_X|\theta, \lambda) \propto \big|\Sigma_\theta\big|^{-1/2}\big|M_X^T\Sigma_\theta^{-1}M_X\big|^{-1/2}q_{\theta, \lambda}^{(-(n-p)/2)},
 \]
hence we are interesting in computing the following for $i \in \{1, 2, \dots, n\}$: 
 \begin{equation}
 \label{eq: det of sub kernel mat}
 \text{Det}^{(-i)} = \bigg\{\big|\Sigma_\theta^{(-i)}\big|, \big|(M_X^{(-i)})^T\Sigma_\theta^{(-i)}M_X^{(-i)}\big|\bigg\}.
 \end{equation}

The perturbed least squares problems and linear systems can be solved independently in $\mathcal{O}(n^3)$ time, hence a naive LOOCV procedure would take $\mathcal{O}(n^4)$ time. However, using matrix decompositions, we can improve this to $\mathcal{O}(n^3)$ total time.

\textbf{Algorithms} Algorithms \ref{algo: btgloocv} and \ref{algo: btgdet} are used for efficiently computing $\big\{\text{TParameters}^{(-i)}\big\}_{i=1}^{n}$ and $\big\{\text{Det}^{(-i)}\big\}_{i=1}^n$ for fixed hyperparameters $(\theta, \lambda)$. The total time complexity is $\mathcal{O}(n^3)$, because the dominant costs are precomputing a Cholesky factorization for a kernel matrix and repeating $\mathcal{O}(n^2)$ operations across $n$ sub-models.

\begin{algorithm}[ht]
\caption{T-Distributions of Sub-Models \label{algo: btgloocv}}
\begin{algorithmic}[1]\label{Algo: loocv}
\STATE \textbf{Inputs}   ${Y} = g_\lambda \circ {f}_X, {M}_X, {K}_X$, $\bx$ \\ \STATE \textbf{Outputs:} $\{m^{(-i)}_{\theta, \lambda}\}_{i=1}^n$, $\{q^{(-i)}_{\theta, \lambda}\}_{i=1}^n$, $\{C_{\theta, \lambda}^{(-i)}\}_{i=1}^n$
\STATE Pre-compute ${R}, {R}_X$, and $\hat\bx$, where ${R}^T{R} = {K}_X$, ${R}_X^T{R}_X = {M}_X^T{K}_X^{-1}{M}_X$, $\hat\bx = {K}_X^{-1}{Y}$
\FOR{$i = 1 \dots n$}
\STATE $\ell_i = K_X^{-1}e_i/|e_i^TK^{-1}e_i|$
\STATE ${R}_X^{(-i)} \leftarrow \text{Downdate}({R}_X, \ell_i)$ \,\,\,\,\,\,\,\,\,\,(\text{Proposition \ref{prop: downdate for matrices}})
\STATE $r_i \leftarrow {Y}_i / \|{R}^{-T}{e}_i\|_2^2$
\STATE $\hat\bx^{(-i)} \leftarrow \hat\bx - r_i{R}^{-1}({R}^{-T}{e}_i)$ \,\,\,\,\,\,\,\,\,\,\,(\text{Proposition \ref{prop: Abridged LS}})
\STATE ${ \beta}_{\theta, \lambda}^{(-i)} \leftarrow \big({R}_X^{(-i)}\big)^{-1} \big({R}_X^{(-i)}\big)^{-T}{M}_X^{(-i)}\hat\bx^{(-i)}$
\STATE ${r}^{(-i)}\leftarrow {Y}^{(-i)} - {M}_X^{(-i)}{ \beta}_{{ \theta}, \lambda}^{(-i)}$
\STATE $\tilde{q}_{\theta, \lambda}^{(-i)} \leftarrow \left\|{r}^{(-i)}\right\|_{{K}_X^{-1}}^2$ 
\STATE $m_{\lambda,\theta}^{(-i)} \leftarrow  
{K}_{xX}\big({R}_X^{(-i)}\big)^{-1} \big({R}_X^{(-i)}\big)^{-T}{r}^{(-i)} + \big({ \beta}_{\lambda,\theta}^{(-i)}\big)^T{m}( x)$
\STATE ${C}_{\theta, \lambda}^{(-i)} \leftarrow B(\bx^{(-i)})/ [k_\theta(\bx^{(-i)}, \bx^{(-i)})]$
\ENDFOR
\RETURN $\{m^{(-i)}_{\theta, \lambda}\}_{i=1}^n$, $\{q^{(-i)}_{\theta, \lambda}\}_{i=1}^n$, $\{C_{\theta, \lambda}^{(-i)}\}_{i=1}^n$
\end{algorithmic}
\end{algorithm}
 
\textbf{Frozen Hyperparameters}
We remark that our LOOCV algorithm is possible because sparse grids and QMC are \textit{deterministic}---since the underlying sampling grids in hyperparameter-space are frozen---in contrast to Monte Carlo (MC) methods, which are \textit{stochastic}. Since we use fixed sparse grids, and we are in fact interested in evaluating the posterior distribution at fixed hyper-parameters $\{\theta_i,\lambda_i\}_{i=1}^M$. If the sampling grid were not frozen across sub-models, our approach would not be viable, because the sampled points in hyperparameter-space would be different for each sub-model. Likewise, in the MLE approach, hyperparameters $\{\theta_i, \lambda_i\}_{i=1}^M$ should theoretically be retrained on the submodels, hence we cannot re-use computed values.

\begin{algorithm}[ht]
\caption{Fast Determinant Computation \label{algo: btgdet}}
\begin{algorithmic}[1]\label{Algo: loocv}
\STATE \textbf{Inputs} $K_X$
\STATE \textbf{Output} $\{\log |K_X^{(-i)}|\}_{i=1}^n$
\STATE Precompute $R^TR = K_X$
\STATE Precompute $\log(\left|K_X\right|)$
\FOR{$i = 1 \dots n$}
\STATE $b_i = e_i^TK_X^{(-1)}e_i$
\STATE $\log|K_X^{(-i)}| \leftarrow \log(|K_X|) + \log(b_i)$ (Propsition \ref{prop: det of principal minor})
\ENDFOR
\STATE \textbf{return} 
\end{algorithmic}
\end{algorithm}

\subsection{Auxiliary Results}\label{subsections: building blocks}
In this section, we present linear algebra results used in the derivations of Algorithms $1$ and $2$ in $\S$ \ref{subsection:derivation}.

\begin{proposition}[Determinant of a Principal Minor\label{prop: det of principal minor}]
\[\det\big(\Sigma^{(-i)}\big) = \det(\Sigma)\big(e_i^T\Sigma^{-1}e_i\big)\]
\end{proposition}

\begin{proposition}[Abridged Linear System\label{prop: Abridged LS}]
Let $K\in \mathbb{R}^{n\times n}$ be of full rank, and let ${c}, {y}\in\mathbb{R}^n$ satisfy ${K}{c} = {y}$. Then if $r_i = c_i / {e}_i^{T}{K}^{-1}{e}_i$, we have:
\[{c}^{(-i)} = ({K}^{(-i)})^{-1}{y}^{(-i)} = {c} - r_i{K}^{-1}{e}_i .\]
\end{proposition}

\begin{lemma}[Determinant of the Schur Complement of a Principal Minor\label{lemma: det of schur complement of principal minor}]
If $X\in\mathbb{R}^{n\times m}$ with $m<n$ has full column rank and $\Sigma\in\mathbb{R}^{n\times n}$ is a positive definite matrix, then
\begin{align*}
&\det\left(\big(X^{(-i)}\big)^T\big(\Sigma^{(-i)}\big)^{-1} X^{(-i)}\right) \\&= -\frac{1}{\det\big(\Sigma^{(-i)}\big)} \det \left(\begin{bmatrix} \Sigma & X\\ X^T&O \end{bmatrix}\right)e_i^T \begin{bmatrix} \Sigma & X\\ X^T & O \end{bmatrix}^{-1} e_i\end{align*}
\end{lemma}
\begin{proof}
Extend the Cholesky factorization $R_{11}^TR_{11}$ of $\Sigma$ to obtain the LDL-decomposition
\[W:=\begin{bmatrix} \Sigma & X\\ X^T & O \end{bmatrix} = \begin{bmatrix} R_{11}^T & 0 \\ R_{12}^T & R_{22}^T \end{bmatrix} \begin{bmatrix} I &0 \\ 0 & -I \end{bmatrix} \begin{bmatrix} R_{11}&R_{12} \\ 0 & R_{22} \end{bmatrix}\]
where $R_{22} = \text{Cholesky}(R_{12}^TR_{12})$ and $R_{12} = R_{11}^{-T}X$. Observe $\big(X^{(-i)}\big)^T\big(\Sigma^{(-i)}\big)^{-1} X^{(-i)}$ is a Schur complement of $W^{(-i)}$. This implies that
\begin{align*}&\det\big(W^{(-i)}\big) \\&= \det\big(\Sigma^{(-i)}\big) \det\left( - \big(X^{(-i)}\big)^T\big(\Sigma^{(-i)}\big)^{-1} X^{(-i)}\right)\end{align*}
By Proposition \ref{prop: det of principal minor}
\begin{align*}
\det\big(W^{(-i)}\big) = \det(W)e_i^TW^{-1}e_i.
\end{align*}
Therefore
\begin{align*}
    &\det\left(\big(X^{(-i)}\big)^T\big(\Sigma^{(-i)}\big)^{-1} X^{(-i)}\right) \\&= \frac{1}{\det(\Sigma^{(-i)})}\det(W)e_i^TW^{-1}e_i.
\end{align*}
\end{proof}

\begin{lemma}[Rank one downdate for bilinear forms\label{lemma: downdate for bilinear forms}]
If $x\in\mathbb{R}^{n}$ and $\Sigma$ is a positive definite matrix in $\mathbb{R}^{n\times n}$, then 
\begin{align*}\big(x^{(-i)}\big)^T \big(\Sigma^{(-i)}\big)^{-1}x^{(-i)} = x^T \left( \Sigma^{-1} - \frac{\Sigma^{-1}e_ie_i^T\Sigma^{-1}}{e_i^T\Sigma^{-1}e_i} \right)x\end{align*}
where $\Sigma^{(-i)}\in\mathbb{R}^{(n-1)\times(n-1)}$ is the $(i, i)th$ principal minor of $\Sigma$ and $x^{(-i)}\in\mathbb{R}^{(n-1)}$ results from deleting the $i$th entry of $x$.
\end{lemma}
\begin{proof}
By Lemma \ref{lemma: det of schur complement of principal minor}, we have
\begin{align*}
    \big(x^{(-i)}\big)^T\big(\Sigma^{(-i)}\big)^{-1}x^{(-i)} &= \det\left(\big(x^{(-i)}\big)^T\big(\Sigma^{(-i)}\big)^{-1}x^{(-i)}\right)
     \\&= \frac{1}{\det(\Sigma^{(-i)})}\det(W)e_i^TW^{-1}e_i.
\end{align*}
In this equation,
\[W = \begin{bmatrix} \Sigma & X\\ X^T & O \end{bmatrix} = R^T \begin{bmatrix} I &0 \\ 0 & -1 \end{bmatrix}  R, \quad R = \begin{bmatrix}R_{11}&R_{12}\\O & R_{22}\end{bmatrix},\] where $R_{11} = \text{chol}(\Sigma)$, $R_{12} = R_{11}^{-T}X$, and $R_{22} = \sqrt{x^T\Sigma^{-1}x}$. Using this decomposition, we may compute the term $e_i^TW^{-1}e_i$. 
Since, 
\[R^{-T}e_i = \begin{bmatrix} R_{11}^{-T}e_i && -\dfrac{x^T\Sigma^{-1}e_i}{\sqrt{x^T\Sigma^{-1}x}} \end{bmatrix}^T, \]
we have
\begin{align*}
    e_i^TW^{-1}e_i &= e_i^TR^{-1}R^{-T}e_i - \frac{(e_i^T\Sigma^{-T}x)(x^T\Sigma^{-1}e_i)}{x^T\Sigma^{-1}x} 
     \\&= e_i^T\Sigma^{-1}e_i - \frac{(x^T\Sigma^{-1}e_i)^2}{x^T\Sigma^{-1}x}.
\end{align*}
Lastly, we have
\begin{align*}
    \frac{\det(W)}{\det(\Sigma^{(-i)})} = \frac{-\det(\Sigma)\det(x^T\Sigma^{-1}x)}{\det(\Sigma)e_i^T\Sigma^{-1}e_i} = \frac{x^T\Sigma^{-1}x}{e_i^T\Sigma^{-1}e_i}.
\end{align*}
These together imply that
\begin{align*}
     \big(x^{(-i)}\big)^T\big(\Sigma^{(-i)}\big)^{-1}x^{(-i)} &= x^T\Sigma^{-1}x - \frac{x^T\Sigma^{-1}e_ie_i^T\Sigma^{-1}e_i}{e_i^T\Sigma^{-1}e_i}
\end{align*}
as desired.
\end{proof}

\begin{proposition}[Rank one matrix downdate\label{prop: downdate for matrices}]
If $X\in\mathbb{R}^{n\times m}$ with $m<n$ has full column rank and $\Sigma$ is a positive definite matrix in $\mathbb{R}^{n\times n}$. Let $v_i = \Sigma^{-1}e_i$ then
\begin{align*}\big(X^{(-i)}\big)^T \big(\Sigma^{(-i)}\big)^{-1}X^{(-i)} = X^T \left( \Sigma^{-1} - \frac{v_iv_i^T}{e_i^T\Sigma^{-1}e_i} \right)X,
\end{align*}
where $\Sigma^{(-i)}\in\mathbb{R}^{(n-1)\times(n-1)}$ is the $(i, i)th$ principal minor of $\Sigma$ and $X^{(-i)}\in\mathbb{R}^{(n-1)\times m}$ results from deleting row $i$ from $X$.
\end{proposition}
\begin{proof}
Let $\hat{\Sigma}:=\Sigma^{-1} - \dfrac{\Sigma^{-1}e_ie_i^T\Sigma^{-1}}{e_i^T\Sigma^{-1}e_i}$. It suffices to prove that 
\begin{align*}
    (x^{(-i)})^T (\Sigma^{(-i)})^{-1}y^{(-i)} = x^T \hat{\Sigma}y,\quad \forall x, y\in\mathbb{R}^N
\end{align*}
However, this follows from Lemma \ref{lemma: downdate for bilinear forms}, because 
\[\big((x+y)^{(-i)}\big)^T \big(\Sigma^{(-i)}\big)^{-1}(x+y)^{(-i)} = (x+y)^T \hat{\Sigma}(x+y)\]
Expanding and canceling symmetric terms yields
\begin{align*}
    &\big(x^{(-i)}\big)^T \big(\Sigma^{(-i)}\big)^{-1}y^{(-i)}+\big(y^{(-i)}\big)^T \big(\Sigma^{(-i)}\big)^{-1}x^{(-i)} \\&= x^T \hat{\Sigma}y + y^T \hat{\Sigma}x
\end{align*}
implying the result.
\end{proof}

\subsection{Algorithm Derivation}
 \label{subsection:derivation}
Recall the following definitions of elements in $\{\text{TParameters}^{(-i)}\}_{i=1}^n$ from \S \ref{fastloocv}:
\begin{equation}
 q_{\theta,\lambda} = \min_\beta \left\| g_\lambda\big({f}_X\big) - {M}_X{ \beta}\right\|^2_{\big({K}_X\big)^{-1}}
\end{equation}
\begin{equation}
 \hat{\beta}_{\theta,\lambda} = \text{argmin}_\beta \left\| g_\lambda\big({f}_X\big) - {M}_X{ \beta}\right\|^2_{\big({K}_X\big)^{-1}}
\end{equation}
 \begin{equation}
m_{\lambda,\theta} =  
{K}_{\bx X}{K}_{X}^{-1}\big(g_\lambda({f}_X) - {M}_X\hat\beta_{\lambda,\theta}\big) + \hat\beta_{\lambda,\theta}^Tm(\bx)
\end{equation}
 \begin{align}
     C_{\lambda, \theta} = B(\bx) / [k_\theta(\bx, \bx)] \,\,\,(\text{Schur Complement})
 \end{align}
We use the generalized least squares LOOCV subroutine, outlined in Section \ref{ssec:ls} to compute $q^{(-i)}_{\theta, \lambda}$ and $\hat{\beta^{(-i)}_{\theta, \lambda}}$ efficiently for all $i\in\{1, ..., n\}$. We use Proposition \ref{prop: Abridged LS} to efficiently compute $m_{\lambda, \theta}$ and $C_{\lambda, \theta}$ whenever a perturbed linear system arises. Generally, these routines involve precomputing a Cholesky decomposition and using it for back-substitution. These steps are enumerated in Algorithm \ref{algo: btgloocv}.
 
The computation of $\{\text{Det}^{(-i)}\}_{i=1}^n$ is straightforward given Proposition \ref{prop: det of principal minor} and a Cholesky decomposition of the kernel matrix. 
 
 \subsubsection{Generalized Least Squares}
 \label{ssec:ls}

The generalized least squares (GLS) LOOCV problem is that of solving the following set of problems efficiently:
\begin{equation*}
\bigg\{\argmin_{x\in\mathbb{R}^p} \left\|b^{(-i)} -  A^{(-i)}x\right\|_{\Sigma^{(-i)}}^2\bigg\}_{i=1}^n.
\end{equation*}
It is assumed that $\Sigma \in \mathbb{R}^{n\times n}$ is positive definite, $b\in\mathbb{R}^{n}$, $A\in\mathbb{R}^{n\times p}$, $\text{Rank}(A) = \text{Rank}(A^{(-i)}) = p$ for some $p<n$ and for all $i \in \{1, 2, ..., n\}$.

We consider the normal equations for the $i$th subproblem:
\begin{align*}
\argmin_{x\in\mathbb{R}} 
\big(A^{(-i)}x - b^{(-i)}\big)^T
\big(\Sigma^{(-i)}\big)^{-1}
\big(A^{(-i)}x - b^{(-i)}\big),
\end{align*}
namely,
\begin{align}
\label{eq: normal equation}
    \big(A^{(-i)}\big)^T\big(\Sigma^{(-i)}\big)^{-1}A^{(-i)}x = \big(A^{(-i)}\big)^T\big(\Sigma^{(-i)}\big)^{-1}b^{(-i)}.
\end{align}
We first show that Equation \ref{eq: normal equation} has a unique solution. By Proposition \ref{prop: downdate for matrices}, we have
\begin{align*}\big(A^{(-i)}\big)^T\big(\Sigma^{(-i)}\big)^{-1}\big(A^{(-i)}\big) = A^T\Sigma^{-1}A -\frac{v_iv_i^T}{e_i^T\Sigma^{-1}e_i}, \end{align*}
where $v_i = A^T\Sigma^{-1}e_i$. The LHS is a rank-1 downdate applied to $A^T\Sigma^{-1}A$. Moreover, the LHS is positive definite and hence invertible, because by assumption, $\text{Rank}(A^{(-i)})=p$, and $\Sigma^{(-i)}$ is positive definite.

We find the solution to Equation \ref{eq: normal equation} by first computing the Cholesky factorization of the LHS. Specifically, given a Cholesky factorization $R^TR$ of $A^T\Sigma^{-1}A$ from the full problem, the Cholesky factorization of the subproblem can be computed by a $\mathcal{O}(p^2)$ Cholesky downdate $R^{(-i)} = \text{Downdate}(R, v_i/e_i^T\Sigma^{-1}e_i)$ such that
\begin{equation*} 
\big(R^{(-i)}\big)^T R^{(-i)} = \big(A^{(-i)}\big)^T\big(\Sigma^{(-i)}\big)^{-1}A^{(-i)}.
\end{equation*}
We therefore can solve the normal equation \ref{eq: normal equation}
\begin{align*}
    x^{(-i)} = \bigg(\big(A^{(-i)}\big)^T\big(\Sigma^{(-i)}\big)^{-1}A^{(-i)}\bigg)^{-1}\big(A^{(-i)}\big)^Ty^{(-i)},
\end{align*}
where where $y^{(-i)} = (\Sigma^{(-i)})^{-1}b^{(-i)}$. The cost of $\mathcal{O}(n^2)$ is attained by evaluating terms from right to left. We first evaluate $y^{(-i)}$ in $\mathcal{O}(n^2)$ time by making use of Proposition \ref{prop: Abridged LS}. We then perform back-substitution using the cholesky factor $R^{(-i)}$ in $\mathcal{O}(p^2)$ time. The overall time complexity is thus $\mathcal{O}(n^3)$. 
 
\section{Experiment Details}
\textbf{Implementation} We run all experiments using our Julia software package, which supports a variety of models (WGP, CWGP and BTG) and allows for flexible treatment of hyperparameters. We also implement several single and composed transformations. For MLE optimization, we use the L-BFGS algorithm from the Julia \verb|Optim| package \citep{mogensen2018optim}.
 
\textbf{Kernel} We used the RBF kernel for all experiments:
\[
    k_\theta(\bm x, \bm x') =\frac{1}{\tau^2} \exp\Big(-\frac{1}{2}\lVert \bm x - \bm x' \rVert^2_{D^{-2}_\theta}\Big) + \sigma^2 \delta_{\bm x \bm x'}.
\]
\textbf{Model Details} To model observation input noise for BTG, we add a regularization term to make the analytical marginalization of mean and precision tractable. We also assume the constant covariate $m(\bx) = 1_n$ in the BTG model, and normalize observations to the unit interval. We assume the constant mean field for both BTG and WGP.

\subsection{Datasets and setups}
\textbf{Two synthetic datasets}: \textbf{IntSine} and \textbf{SixHumpCamel}. The \textbf{IntSine} dataset, also used by \cite{BWGP}, is sampled from a rounded 1-dimensional sine function with Gaussian noise of a given variance. The training set is comprised of 51 uniformly spaced samples on $[-\pi, \pi]$. The testing set consists of 400 uniformly spaced points on $[-\pi, \pi]$. The \textbf{SixHumpCamel} function is a 2-dimensional benchmark optimization function usually evaluated on $[-3, 3]\times [-2, 2]$ \citep{sixhumpcamel}. We shift its values to be strictly positive. The training set is comprised of $50$ quasi-uniform samples, i.e., a 2-dimensional Sobol sequence, on $[-1, 1]\times [-2, 2]$. The testing set consists of 400 uniformly distributed points on the same domain.

\textbf{Three real datasets}: \textbf{Abalone}, \textbf{WineQuality} and \textbf{Creep}.  \textbf{Abalone} is an 8-dimensional dataset, for which the prediction task is to determine the age of an abalone using eight physical measurements \citep{Dua:2019}. The \textbf{WineQuality} dataset has 12-dimensional explanatory variables and relates the quality of wine to input attributes \citep{wine_dataset}. The \textbf{Creep} dataset is 30-dimensional and relates the creep rupture stress (in MPa) for steel to chemical composition and other features \citep{creep_dataset}. To simulate data-sparse training scenarios, we randomly select training samples of size 30, 200, and 100 from \textbf{Abalone}, \textbf{WineQuality}, and \textbf{Creep}, respectively, and test on 500, 1000 and 1500 out-of-sample points.
\subsection{Performance Metrics}
We use two loss functions to evaluate model performance: root mean squared error (RMSE) and mean absolute error (MAE). Let $\{f^*(\bm{x_i})\}_{i=1}^n$ be true test labels and $\{\hat{f}(\bm{x_i})\}_{i=1}^n$ be predictions, which are taken to be predictive medians in WGP and BTG. The predictive median is also used by \citet{WGP}.
\begin{equation*}
\begin{split}
    \text{RMSE} & = \left( \frac{1}{n}\sum_{i=1}^n (f^*(\bm{x_i}) - \hat{f}(\bm{x_i}))^2  \right)^{\frac{1}{2}}, \\
    \text{MAE} & = \frac{1}{n}\sum_{i=1}^n \big| f^*(\bm{x_i}) - \hat{f}(\bm{x_i})\big|,
\end{split}
\end{equation*}

\begin{table}[!t]
\centering
\caption{Quantile Timing. Time (s) for finding predictive median and credible intervals using no quantile bound, convex hull quantile bound, and singular weight quantile bound. Results are averaged over $10$ trials. \label{table:quantile timing}}
\begin{tabular}{lccc}
\hline
\textbf{} & \multicolumn{1}{c}{Total} & \multicolumn{1}{c}{Median} & \multicolumn{1}{c}{CI} \\ \hline
N/A    & 13.0  & 3.24  &  7.87    \\
Convex Hull  &   6.21   & 1.11  &3.19   \\
Single Weight    & 11.0   & 2.52  & 6.54    \\
\hline
\end{tabular}
\end{table}

\begin{table*}[!t]
\centering
\caption{Regression Timing Results for SixHumpCamel, Abalone and Wine datasets: dimension (Dim) and time cost (min). Models: GP, WGP, CWGP, BTG. Transformations: SA:SinhArcSinh, BC:BoxCox, L:affine. \label{table: regression timing}}
\begin{tabular}{lcccccc}
\hline
          & \multicolumn{2}{c}{\textbf{SixHumpCamel}}     & \multicolumn{2}{c}{\textbf{Abalone}} & \multicolumn{2}{c}{\textbf{WineQuality}}                 \\ \hline
\textbf{} & \multicolumn{1}{c}{Dim} & \multicolumn{1}{c}{Time (min)} & \multicolumn{1}{c}{Dim} & \multicolumn{1}{c}{Time (min)} & \multicolumn{1}{c}{Dim} & \multicolumn{1}{c}{Time (min)} \\ \hline
WGP-BC    & 5  & 0.68  & 10  & 1.28  & 14  & 1.52  \\
WGP-SA    & 6   & 0.78  & 11  & 1.22  & 15  & 1.60  \\
CWGP-L-SA & 8   & 1.08  & 13  & 1.48  & 17  & 2.50  \\
CWGP-A-BC & 9   & 1.14  & 15  & 1.56  & 18  & 2.82  \\
BTG-BC    & 4   & 1.10  & 9   & 1.02  & 13  & 1.40  \\
BTG-SA    & 5   & 0.95  & 10  & 1.04  & 14  & 1.29  \\
BTG-L-SA  & 7   & 1.74  & 12  & 1.11  & 16  & 1.79  \\
BTG-A-BC  & 8   & 1.65  & 14  & 1.09  & 17  & 1.87 \\\hline
\end{tabular}
\end{table*}

\subsection{Scaling Experiments}
\textbf{Sparse Grids vs QMC}
We compare sparse grid and QMC quadrature rules under our quadrature sparsification framework. We use the \textbf{SixHumpCamel} dataset with 30 training data points and 100 testing data points as a toy problem. We train BTG with the composed transformation Affine-SinhArcSinh. The hyperparameter space is $7$-dimensional.

\textbf{Quantile Bounds Speedup}
To assess the effectiveness of quantile bounds, we set up a problem using the \textbf{Levy1D} dataset, $200$ training points, and a QMC grid with $50$ nodes. The ftol for Brent's algorithm is $10^{-3}$. Table \ref{table:quantile timing} shows detailed timing results.

\subsection{Timing Details For Regression Experiments}
We compare the end-to-end inference time of BTG and WGP for a range of datasets. The time cost of BTG depends mostly on the dimension of the integral, i.e., the total number of hyperparameters that we must marginalize out in the fully Bayesian approach. We report timing results for a representative set of regression experiments with dimensions ranging from 4 to 18. 

As is shown in Table \ref{table: regression timing}, the time cost of BTG is slightly larger than WGP in low dimensional problems (dimension less than 10), while in higher dimensional problems (dimension larger than 10) BTG can be even faster than WGP. Generally, we conclude that the end-to-end inference speed of BTG is comparable to WGP.

\end{document}